\documentclass{article}

\usepackage{arxiv}

\usepackage[utf8]{inputenc} 
\usepackage[T1]{fontenc}    
\usepackage{hyperref}       
\usepackage{url}            
\usepackage{booktabs}       
\usepackage{amsfonts}       
\usepackage{nicefrac}       
\usepackage{microtype}      
\usepackage{lipsum}		
\usepackage{graphicx}
\usepackage{natbib}
\usepackage{doi}
\usepackage{multirow}
\usepackage{times}  
\usepackage{helvet} 
\usepackage{courier}
\usepackage{amsmath,amsthm, amssymb}
\usepackage{algorithm, algorithmic}
\newtheorem{theorem}{Theorem}[section]
\def\E{\mathbb{E}}
\newcommand{\algrule}[1][.2pt]{\par\vskip.5\baselineskip\hrule height #1\par\vskip.5\baselineskip}
\newcommand{\tabincell}[2]{\begin{tabular}{@{}#1@{}}#2\end{tabular}}

\title{Model Rectification via Unknown Unknowns \\Extraction from Deployment Samples}

\author{
Bruno Abrahao$^{1,2}$\thanks{Equal contribution and listed in alphabetical order. Correspondence to: \texttt{<abrahao@nyu.edu>} or \texttt{<zw1454@nyu.edu>}}\\
\And
Zheng Wang$^1$\footnotemark[\value{footnote}] \\
\And
Haider Ahmed$^1$ \\
\And
Yuchen Zhu$^1$ \\
}
\date{%
    $^1$New York University Shanghai\\%
    $^2$New York University\\[2ex]%
} 					

\renewcommand{\shorttitle}{Model Rectification via Unknown Unknowns Extraction from Deployment Samples}

\hypersetup{
pdftitle={Model Rectification via Unknown Unknowns Extraction from Deployment Samples},
pdfauthor={Bruno Abrahao, Zheng Wang, Haider Ahmed, Yuchen Zhu},
}

\begin{document}
\maketitle

\begin{abstract}
Model deficiency that results from incomplete training data is a form of structural blindness that leads to costly errors, oftentimes with high confidence. During the  training of classification tasks, underrepresented class-conditional distributions that a given hypothesis space can recognize results in a mismatch between the model and the target space. To mitigate the consequences of this discrepancy, we propose \emph{Random Test Sampling and Cross-Validation} (RTSCV) as a general algorithmic framework that aims to perform a post-training model rectification at deployment time in a supervised way. RTSCV extracts unknown unknowns (u.u.s), i.e., examples from the class-conditional distributions that a classifier is oblivious to, and works in combination with a diverse family of modern prediction models. RTSCV augments the training set with a sample of the test set (or deployment data) and uses this redefined class layout to discover u.u.s via cross-validation, without relying on active learning or budgeted queries to an oracle. We contribute a theoretical analysis that establishes performance guarantees based on the design bases of modern classifiers. Our experimental evaluation demonstrates RTSCV's effectiveness, using 7 benchmark tabular and computer vision datasets, by reducing a performance gap as large as 41\% from the respective pre-rectification models. Last we show that RTSCV consistently outperforms state-of-the-art approaches.
\end{abstract}

\section{Introduction}
Data quality constitutes a critical factor affecting the performance of prediction models. In particular, incomplete training data frequently results in structural mismatches between data-driven trained models and the respective target space in which they are supposed to be deployed, which makes most classifiers susceptible to systematic errors due to their limited ability to rectify a model post-training.

In scenarios of increasing dependence on  algorithmic decisions in high stake situations, deficient models result in costly (sometimes fatal) errors, unfairness, and other problems. For example, an autopilot system may fail to recognize peculiar traffic signs it has never encountered during training, leading to accidents. In the case of automated recruiting, data from industries dominated by a given gender may result in biased classifiers, likely to reject examples of the opposite gender due to the lack of enough successful observations that belong to that class. In addition, the unseen joint distribution of features and ``data-drift'' may contribute to high confidence errors. For instance, when training a classifier to distinguish between white dogs and black cats, when presented with a black dog at deployment time, the model may predict ``cat'' with high confidence~\citep{Lakkaraju2017}. For ``data-drift,'' the structure of the target space may change over time and deviate from the trained model. Take, for example, the anecdotal account in the beginning of the COVID-19 pandemic, where physicians attempted to identify what type of ``bacteria'' had been causing an unusual high number of ``pneumonia'' cases, overlooking the fact that there was a new type of agent, i.e., a novel virus affecting the respiratory system.

We focus on classification tasks where a trained model may be oblivious to some of the domain-specific class-conditional distributions that a set of hypotheses can recognize. Data examples from these ``invisible'' joint-distributions of features that a classifier is oblivious to, i.e., ``hidden classes'', form the \emph{unknown unknowns} (u.u.s), which cause a prediction model to make errors with high confidence. This definition encompasses other terms researchers use in different contexts. For instance, in the literature that addresses over-confident softmax predictions of neural networks, especially in computer vision, the term \emph{out-of-distribution} (OOD) samples refer to the same concept~\citep{2017OOD-baseline, ODIN-OOD, MD-OOD, energy-OOD}. In addition, researchers have name the problem of classification with u.u.s the \emph{Open Set Recognition} (OSR) problem~\citep{OpenSet, WSVM}, due to the contrast between the ``open" nature of discovering u.u.s and the traditional closed set scenario, where the training and test classes match. 

We contribute to the mitigation of the u.u.s problem by proposing \emph{Random Test Sampling and Cross-Validation} (RTSCV)\footnote{For reproducibility, our code will be made publicly available and we will replace this footnote with the GitHub link after the anonymous reviewing phase.}, a general algorithmic framework that aims to perform a post-training model rectification of a base classifier at deployment time in a supervised way. RTSCV aims to reduce the structural mismatch between a trained model and the target space by extracting u.u.s from samples of the target space. RTSCV augments the training set with a sample of the test set (or deployment data) and uses this redefined class layout to discover unknown unknowns via cross-validation. Our key insight is that by augmenting a training set that possesses $m$ classes with a dummy class, labeled $m+1$, whose examples come from a test set sample, cross-validation is likely to decouple examples that belong to known classes from $m+1$, due to the high variance and broad boundary of this dummy class. Conversely, u.u.s. coming from separable classes may share greater affinity to the dummy class, as they are expected to being poor fits to the known classes and because the decision boundary around the dummy class may have been established with the contribution of examples from the u.u.s in the test sample.

RTSCV bears two advantages compared to previous methods. First, RTSCV can work in combination with a diverse family of modern classifiers. The bulk of existing methods on identifying u.u.s focus on modifying specific classification methods, such as SVM, KNN, and DNNs, in such a way as to include a free parameter that can be learned at the deployment phase. This allows for the method to predict u.u.s as possible outputs~\citep{OpenSet,WSVM,Jnior2016NearestKNN,OpenMax}. However, unlike RTSCV, these methods do not generalize, as they are classifier-specific. We note that RTSCV can be easily paired with any trained classifier. Second, RTSCV relies solely on the use of a sample of the test data, which removes assumptions made by approaches like active learning, which are often challenging to operationalize in practice, such as the existence of budgeted queries to an oracle~\citep{contradict-the-machine,Lakkaraju2017,MachineTeaching}.

We contribute a theoretical analysis with performance guarantees based on the design bases of modern classifiers, including Maximum Likelihood Estimation, Bayes classifier, and Minimum Mahalanobis distance. Through an extensive experimental evaluation, we use 7 benchmark tabular and more challenging computer vision datasets like CIFAR-10, CIFAR-100, and SVHN with ResNet and DenseNet as base models. Our results suggest that RTSCV is a promising direction for post-training rectification of a base classifier by reducing a performance gap (Accuracy, F-measure, AUROC) as large as 41\%. Moreover, our results indicate that RTSCV consistently outperforms state-of-the-art approaches and baselines by a significant margin.

\subsection{Prediction Errors and Assumptions}
The conceptual idea of a ``class'' is often subjective, and different hypothesis spaces will separate the feature space into different class-conditional distributions. Here we employ a working definition of u.u.s classes via a geometric argument. That is, given a fixed hypothesis space, the u.u.s form separable clusters in the feature space that are distinguishable from the known structures in the target space. This definition is without loss of generality, as it allows for any abstraction of conceptual blindness to examples. 
We note that u.u.s are not the only sources of prediction errors. To delineate the aims of RTSCV, here we discuss different types of errors and the scope in which RTSCV operates. Let $f$ be a classifier and consider the hypothesis space produced by this model, i.e., the set of all functions that can be returned by it. We assume that the model is consistent, that is, if there is a function in the hypothesis space, the machine is going to produce that function from training. Further, let $E$ be the \emph{Bayes Error}, or the irreducible error. If $E(\mathcal{H})$ is the lowest error we could produce with hypothesis space $\mathcal{H}$, and $E(\mathcal{H}, D)$ is the minimum error we produce with $\mathcal{H}$ and available training data $D$, then $E(\mathcal{H}, D)-E$ represents the overall \emph{generalization error} given $\mathcal{H}$ and $D$, which can be decomposed as the sum of $E(\mathcal{H}, D)-E(\mathcal{H})$ and $E(\mathcal{H})-E$. We call the first difference \emph{estimation error}, and the second difference \emph{approximation error}. That is, the model may produce errors due to either deficiencies of the model (approximation error) or to the training data (estimation error), such as u.u.s. In this paper, we focus on the latter, i.e., reducing the estimation error under the assumption of a fixed hypothesis space. We also assume that the data are free of mislabeling errors.

We emphasize the distinction between u.u.s detection and outlier detection. Outliers are rare extreme values, produced by the realization of (possibly known) class-conditional distributions. As outliers tend to be isolated from any cluster in the feature space, we make a distinction with u.u.s, which are exemplary of clusters generated by some joint-distribution of their features, but whose structure is invisible to the trained model. As such, detecting outliers is beyond the scope of our work. In addition, related to our approach is zero-shot learning, where we assume that the test set also includes unseen data classes during training. It aims to discover u.u.s, by creating unseen joint-distributions of features that come from prescribed combinations of attributes among known classes. This method is not comparable to our approach, as we do not make use of such side information.

\subsection{Related Work}
Early work focused on extending classical machine learning algorithms to enable u.u.s prediction. Prominent examples are SVM-based methods, such as \citep{OpenSet}, which proposed the 1-vs-Set machine that separates the feature space with an additional hyperplane parallel to the hyperplane obtained from the SVM. It then optimizes the open space risk for this linear kernel slab model. To further reduce the open set risk, \citep{WSVM} proposed the W-SVM to incorporate non-linear kernels under a compact abating probability (CAP) model. Another similar approach is the $P_I$-SVM by \citep{PISVM}. Besides modifications to SVM, \citep{Jnior2016NearestKNN} introduced the OSR version of the Nearest Neighbor classifier (OSNN) based on a threshold method that relies on measurements of the distance of an u.u.s sample from the known space. 

To address large and high-dimensional datasets, recent approaches proposed to modify Deep Neural Networks (DNNs). A baseline was proposed by \citep{2017OOD-baseline}, formalizing the observation that the softmax predictions may assign high confidence to erroneously classified out-of-distribution samples (u.u.s). \citep{ODIN-OOD} designed the ODIN detector that could better differentiate the confidence scores between in-distribution and out-of-distribution samples in the target space, by combining temperature scaling and input perturbation. Using probabilistic modeling, specifically, Gaussian discriminant analysis (GDA), \citep{MD-OOD} modeled the softmax outputs of known classes as class-conditional Gaussian distributions. It then used the closest Mahalanobis distance (MD) to these Gaussian distributions of each test sample as the confidence score. As a modification to this approach, \citep{LH-OOD} replaced the MD confidence score with the class-conditional log-likelihood. For DNNs, researchers focused on changing the network architecture to adapt to u.u.s detection, such as OpenMax~\citep{OpenMax}, CROSR~\citep{CROSR}, and C2AE~\citep{C2AE}. The surveys by \citep{Geng2020Survey} and \citep{2019Survey} provide a comprehensive review of these methods. For all of the preceding approaches, we argue that the modification of existing classifiers is model-specific. In contrast, RTSCV is a general algorithmic framework capable of working with any classifier.

A different line of research addressed the u.u.s in an incremental or active learning manner. \citep{EVM} formulated a theoretically sound classifier, the Extreme Value Machine (EVM), grounded on the Extreme Value Theory, which is able to perform nonlinear kernel-free variable bandwidth incremental learning. \citep{contradict-the-machine} and \citep{Lakkaraju2017} both proposed a hybrid framework combining human crowdsourcing and algorithmic methods, in which some priors of the u.u.s are extracted by experts whose feedbacks guide adjustments of the trained model. By adopting an active learning environment, these approaches can potentially cope with a dynamic feature space. Nevertheless, requiring the presence of an oracle is oftentimes unrealistic, as it may be time and human-labor expensive, and therefore non-scalable for many real-world applications. Our proposal offsets these shortcomings by relying only on the analysis of a data sample at deployment time.

The idea of using cross-validation to rectify incorrect data was previously experimented to identify mislabeled training data \citep{Brodley1999}. In their noise-reduction approach, cross-validation is performed over the training set and mislabeled data are those given different ``pseudo-labels" from their original labels after the cross-validating phase. While their work focuses on \textit{identifying} mislabeled data, RTSCV aims to augment a model with new labels to the data examples that class-conditional distributions that were not contemplated at training time generates, which allows for the correct classification of these examples.

\begin{algorithm}[h!]
\caption{Random Test Sampling and Cross-Validation}
\label{alg:RTSCV}
\begin{algorithmic}
\STATE {\bf Input: }Training set $X$ with labels $\{1,2,\ldots,m\}$, test set $Y$, sample rate c, base classifier $f$, number of cross-validation folds $k$
\algrule
\STATE 1. Randomly sample test set $Y$ to obtain a subset $X_s$ such that $|X_s| = c \cdot |Y|$
\STATE 2. Assign label $m+1$ to $X_s$
\STATE 3. Obtain an augmented training set $\widetilde{X} \leftarrow X \cup X_{s}$ with labels $\{1,2,\ldots,m+1\}$
\STATE 4. Run a $k$-fold cross-validation on $\widetilde{X}$
\STATE 5. Let $X_u$ be the set of samples with predicted label $m+1$ during cross-validation
\STATE 6. Label samples in $X_u$ with label $m+1$
\STATE 7. Obtain the rectified training set $\overline{X} \leftarrow X \cup X_u$ with labels $\{1,2,\ldots,m+1\}$
\STATE 8. Train $\widehat{f}$ on $\overline{X}$
\STATE 9. \textbf{return} $\widehat{f}$
\end{algorithmic}
\end{algorithm}

\section{The RTSCV Framework}
\label{sec_rtscv}
We set our scope on multi-class classification tasks. Let $f$ be the input base classifier, which we will treat as a black-box. Let $X=\{X_1,X_2, \cdots,X_m\}$ be the training set with labels $\{1,2,\ldots,m\}$. Let $Y$ be the test set (as a representative of the target space).
To detect the u.u.s and rectify a trained model, we present RTSCV in Algorithm \ref{alg:RTSCV}. In summary, we first randomly sample the test set $Y$ of cardinality $|c\cdot Y|$, for a given fraction $c$, to obtain a sample $X_s$, from which we create a new dummy class and assign the label $m+1$. Note that $X_s$ may contain examples from both the known and (potentially multiple) u.u.s classes. We then augment the original training set with examples from $X_s$, resulting in a new intermediate training set $\widetilde{X}$, to which we apply cross-validation in combination with a base classifier $f$.

Intuitively, RTSCV relies on the correct re-classification of samples in $X_s$ during cross-validation, regarding whether they belong to a known or u.u.s class. Samples from the test set make up a high variance class $X_s$, whose boundary encompasses all other classes (i.e., it may contain examples from any class). In light of this, examples that belong to known classes are likely placed in the correct classes due to the conciseness and specificity of the representation. On the contrary, u.u.s are classified as members of $m+1$ due to the dissimilarity with all other classes and to affinity with some examples that contributed to the position of the decision boundary around class $m+1$. The examples classified during cross-validation with label $m+1$ make up a new set $X_u$, (the u.u.s class), which we adjoin to the original training set $X$ to form the rectified model $\overline{X}$.

The sample rate $c$ is a critical hyperparameter. For $c$, a very small sample rate may not contain enough representatives of u.u.s due to a small test set sample, whereas a large $c$ may lead to a sample class $X_s$ that over-represents the structure of the known classes, thereby causing the cross-validation to assign examples of known classes to $X_u$. In Figure~\ref{fig:sample_rate} we illustrate the classifier's performance versus the sample-training ratio, i.e., $c$ as a function of the training set size, for two of the benchmark datasets we used in our experimental evaluation. We discuss the search for the optimal $c$ in Section~\ref{sec_experiment}.

The number of cross-validation folds $k$ determines the running time of RTSCV, whose time complexity is roughly $O(k\cdot T_f)$, where $T_f$ is the running time of the input base classifier $f$ without RTSCV.
Figure~\ref{fig:ablation} (a) displays the model performance against $k$ for several datasets we used in the evaluation of RTSCV. Note that RTSCV effectively rectifies a trained model even if we use the more computationally economical ``holdout validation" approach.

\begin{figure}[t]
\centering
\includegraphics[width=0.49\textwidth]{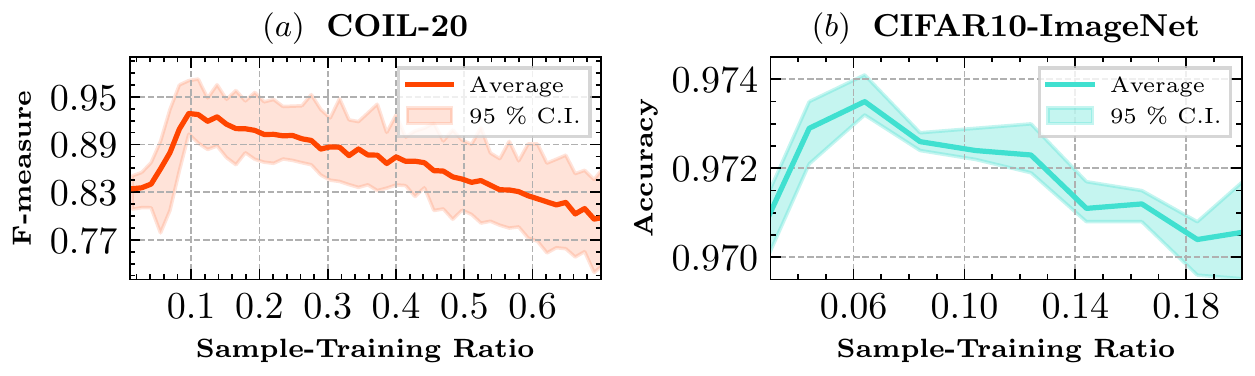}
\caption{The influence of sample-training ratio on RTSCV's performance. Multiple results were recorded to calculate the 95\% confidence interval. Plot (a) corresponds to openness 9.3\%. For (b), the accuracy on the y-axis stands for the combined, overall classification accuracy.}
\label{fig:sample_rate}
\end{figure}

\begin{figure}[t]
\centering
\includegraphics[width=0.98\textwidth]{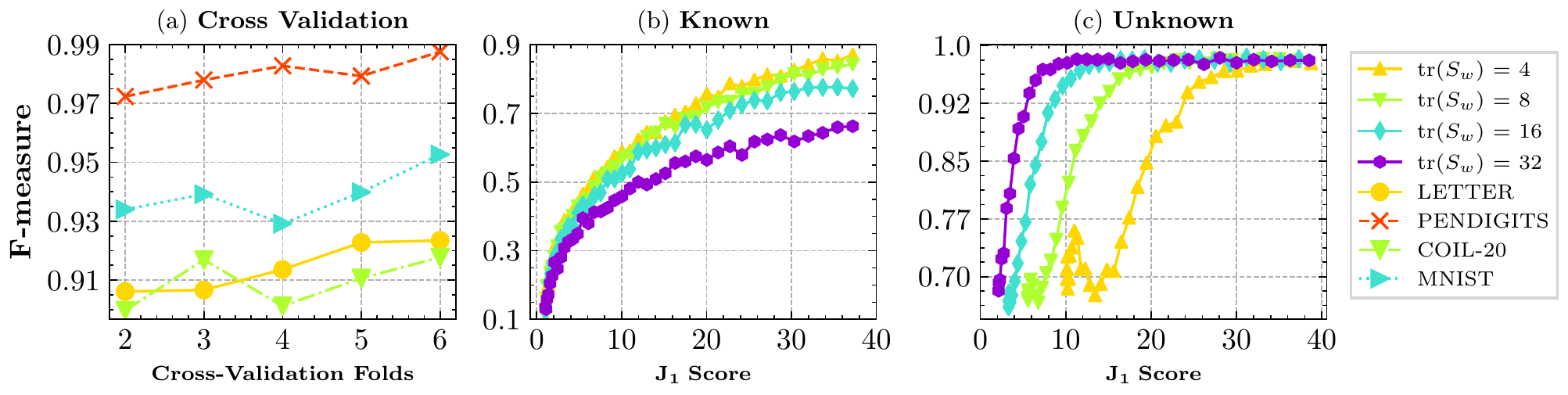}
\caption{(a) The influence of the number of cross-validation folds on RTSCV framework, evaluated on tabular datasets. (b) The influence of known class separability on RTSCV framework. We vary the between-class distances and the covariances of the known classes to generate different $J1$ scores.
(c) The influence of u.u.s class separability on RTSCV framework. We vary the covariance and the distance of the u.u.s class to the known classes.}
\label{fig:ablation}
\end{figure}

\section{Probabilistic Analysis} 
\label{sec_performance_guarantees}
We contribute a theoretical analysis that justifies the correctness of RTSCV and establishes performance guarantees as a function of class separability and the test sample size. As RTSCV is a general framework that may be combined with any base classifier, each employing disparate approaches to establish decision boundaries, there are major challenges in establishing a concise set of mathematical tools that would cover the basis of each specific approach. In light of this, we analyze its behavior through the lens of objectives that modern classifiers aim to optimize to find the sufficient conditions for the correct relabeling of the test set sample $X_s$, namely Maximum Likelihood Estimation (MLE), Bayes classifier (BC), and Minimum Mahalanobis Distance (MMD). These objectives are shared by many classification models, such as the rule-based, margin-based, etc.
 
We structure the following theorems based on two classification cases under RTSCV of a data point $x \in X_s$: (1) the true label of $x$ is one of the known classes, where the correct decision is to assign the true known label during cross-validation, and (2) $x$ belongs to the u.u.s class and the process should keep it in $X_s$. We aim to establish the correctness of RTSCV by showing that this correct decision is the one that optimizes the MLE, the BC, and the MMD. 

We model all the known and u.u.s classes as non-identical multivariate Gaussian distributions. Specifically, let $X_1,X_2,\ldots, X_m$ be the known classes with distinct mean $\mu_i \in \mathbb{R}^d$ and diagonal covariance matrix $\Sigma_i \in \mathbb{R}^{d \times d}$, and $X_u$ is the u.u.s class with mean $\mu_u \in R^d$ and diagonal covariance matrix $\Sigma_u \in \mathbb{R}^{d \times d}$. Furthermore, we assume all the distributions are homoscedastic, so $\Sigma_i = \sigma_i^2 I$ for $i \in \{1,2,\ldots,m\}$ and $\Sigma_u = \sigma_u^2 I$, for some $\sigma_i \in \mathbb{R}^+, \sigma_u \in \mathbb{R}^+$. In this way, since the sample class $X_s$ is obtained by randomly sampling the test set by RTSCV, we can model $X_s$ as a Gaussian mixture of $X_1,X_2,\ldots,X_m,X_u$ weighted by their respective percentage in the test set, denoted $P(X_i), i\in \{1,2,\ldots, m,u\}$. 

Under the preceding assumptions, via Gaussian discriminant analysis, we first focus on MLE to explore how the total likelihood of the dataset changes under different labeling scheme of the sample class $X_s$. For a test set sample $x\in X_s$, its likelihood of being in one of the $X_i$ $(i \in \{1,2,\ldots,m,u\})$ is:
\begin{equation}
L_i(x) = \frac{1}{(2\pi)^{\frac{d}{2}}\sqrt{|\Sigma_i|}}\exp(-\frac{1}{2}(x-\mu_i)^T \Sigma_i^{-1}(x-\mu_i))
\end{equation}
where $|\Sigma_i|$ denotes the determinant of $\Sigma_i$. Its likelihood of being in the sample class $X_s$ is
\begin{equation}
L_s(x)=P(X_u)L_u(x) + \sum_{k=1}^m P(X_k) L_k(x)
\end{equation}
as $X_s$ is a Gaussian mixture. We can now find sufficient conditions under which the correct classification of $x$ can increase the total likelihood.

\begin{theorem}[Maximum Likelihood Estimation]
\label{thm:MLE}
For $x \in X_s$, if $x$ is a sample of some known class, i.e., $x \sim X_k$ for some $k \in \{1,2,\ldots,m\}$, for $x$ to be correctly classified as belonging to $X_k$ based on MLE, we require it to have a higher class-conditional likelihood for class $X_k$ than that of $X_s$. And we have
$$\mathbb{E}_{x\sim X_k}(L_k(x)) \geq \mathbb{E}_{x\sim X_k}(L_s(x))$$
given that, for all $i \in \{1,\cdots,m,u\}$:
$$\|\mu_{k}-\mu_{i}\|_2^{2} \geq d\cdot (\sigma_{k}^{2}+\sigma_{i}^{2}) \cdot \ln \bigg(\frac{2\sigma_{k}^{2}}{\sigma_{k}^{2}+\sigma_{i}^{2}}\bigg)$$
Similarly, if $x \sim X_u$, then we have
$$\mathbb{E}_{x\sim X_u}(L_s(x)) \geq \mathbb{E}_{x\sim X_u}(L_k(x))$$
given that
$$\frac{\|\mu_{k}-\mu_{u}\|_2^{2}}{\sigma_{k}^{2}+\sigma_{u}^{2}} \ge 2\ln \bigg(\frac{1-P(X_k)}{P(X_u)}\bigg)+ d \cdot \ln \bigg(\frac{2\sigma_{u}^{2}}{\sigma_{k}^{2}+\sigma_{u}^{2}}\bigg)$$
\end{theorem}
\begin{proof}
All proofs are in the supplementary materials.
\end{proof}

Theorem \ref{thm:MLE} characterizes sufficient conditions for the re-classification of $X_s$ to maximize total likelihood. In summary, it says that $X_s$ will be correctly classified based on MLE once the class separability is above a threshold characterized by 
the squared difference of the means and the squared within-class variances.

We now turn our attention to the Bayes classifier, which determines the membership of $x \in X_s$ by considering its posterior probability \citep{naiveBayes}. By Bayes' theorem, the posterior for $x$ to be in class $X_i$ is given by $p(X_i|x)=L_i(x)P'(X_i)/p(x)$. In scenario we consider, $P'(X_i)$ stands for the prior of $X_i$ in the training set:
$P^{\prime}(X_i) = |X_i|/{|X_s \cup \bigcup_{k = 1}^{m}X_k|}, i \in \{1,2,\ldots,m,s\}$. In this way, one can calculate the Bayesian decision boundary of $x$ between $X_s$ and $X_k$ for some $k \in \{1,2,\ldots,m\}$.

\begin{theorem}[Bayesian Classification]
\label{thm:Bayesian}
For $x \in X_s$, if $x \sim X_k$, we have:
$$\mathbb{E}_{x\sim X_k}(p(X_k|x)) \geq \mathbb{E}_{x\sim X_k}(p(X_s|x))$$
given that, for all $i \in \{1,\cdots,m,u\}$:
$$\frac{\|\mu_{k}-\mu_{i}\|_2^{2}}{\sigma_{k}^{2}+\sigma_{i}^{2}} \geq 2\ln \bigg(\frac{P^{\prime}(X_s)}{P^{\prime}(X_k)}\bigg)+ d \cdot \ln \bigg(\frac{2\sigma_{k}^{2}}{\sigma_{k}^{2}+\sigma_{i}^{2}}\bigg)$$
If $x \sim X_u$, then we have
$$\mathbb{E}_{x\sim X_u}(p(X_{s}|x)) \geq \mathbb{E}_{x\sim X_u}(p(X_{k}|x))$$
given that 
\begin{align*}
\frac{\|\mu_{k}-\mu_{u}\|_2^{2}}{\sigma_{k}^{2}+\sigma_{u}^{2}} &\geq 2\ln \bigg(\frac{P^{\prime}(X_k)- P^{\prime}(X_s)P(X_k)}{P^{\prime}(X_s)P(X_u)}\bigg) + d \cdot \ln \bigg(\frac{2\sigma_{u}^{2}}{\sigma_{k}^{2}+\sigma_{u}^{2}}\bigg)
\end{align*}
\end{theorem}

Theorem~\ref{thm:Bayesian} (similarly to Theorem~\ref{thm:MLE}) establishes a performance guarantee on the correct behavior of the framework. That is, when the class separability between the u.u.s class and the known classes is greater than a given threshold,
controlled by the sample size, the priors $P'(X_s)$, and the dimensionality of the data, the re-classification of $X_s$ will be as we expect with high probability.

Last, the Minimum Mahalanobis Distance (MMD) mimics the goal shared by classifiers to place a data point in the class whose joint-probability distribution of features is the closest to that of the data point. We show that the correct decision by RTSCV is the one that minimizes the MD.

For a test set sample $x \in X_s$, its squared MD to some class $X_i$ is given by
\begin{equation}
D_M^{2}(x, X_i)=(x-\mu_i)^T \Sigma_i^{-1} (x-\mu_i)
\end{equation}

The analogous behavior of a classifier that employs MMD as a metric would assign $x$ to the class that has the closest MD to $x$\footnote{Note that we assume $\Sigma_s$ to be diagonal to avoid long derivations resulted from the Gaussian mixture model.}.

\begin{theorem}[Minimum Mahalanobis Distance]
\label{thm:MD}
Assume $\Sigma_s = \sigma_{s}^{2}I$. For $x \in X_s$, if $x \sim X_k$ for some $k \in \{1,2,\ldots,m\}$, we have:
$$\mathbb{E}_{x\sim X_k}(D_M^{2}(x, X_k)) \leq \mathbb{E}_{x\sim X_k}(D_M^{2}(x, X_s))$$
given that
$$\|\mu_{k} - \mu_{s}\|_2^{2} \geq d \cdot \sigma_{s}^{2} \cdot \Big(1 - \frac{\sigma_{k}^{2}}{\sigma_{s}^{2}}\Big)$$
If $x \sim X_u$, then we have
$$\mathbb{E}_{x\sim X_u}(D_M^{2}(\boldsymbol{x}, X_s)) \leq \mathbb{E}_{x\sim X_u}(D_M^{2}(\boldsymbol{x}, X_k))$$
given that
$$\frac{\|\mu_{u}-\mu_{k}\|_2^{2}}{\sigma_{k}^{2}} - \frac{\|\mu_{u}-\mu_{s}\|_2^{2}}{\sigma_{s}^{2}} \geq d \cdot \Big(\frac{\sigma_{u}^{2}}{\sigma_{s}^{2}} - \frac{\sigma_{u}^{2}}{\sigma_{k}^{2}}\Big)$$
\end{theorem}

In alignment with Theorem~\ref{thm:MLE} and Theorem~\ref{thm:Bayesian}, here we establish a requirement of class separability as the sufficient condition for the expected behavior of RTSCV.

\subsection{Empirical Illustration} 
To illustrate the preceding theoretical analysis, here we empirically study the effect of class separability using a synthetic dataset that consists of 10 distinct known classes and one u.u.s class, all sampled from pre-fixed 2-dimensional Gaussian distributions. To measure class separability, we adapted the notion of \textit{scatter matrices}~\citep{PatternRecognition} to address the presence of the u.u.s class . Specifically, assume $\{X_1, X_2, \ldots, X_m\}$
are the $m$ known classes with $\mu_i$ being the mean of $X_i$ and $\Sigma_i$ being the covariance matrix of $X_i$. For a u.u.s class $X_u$, its mean and covariance matrix are $\mu_u$ and $\Sigma_u$. The \textit{between-class scatter matrix} $S_b$ is defined to measure the separability between different classes: $S_b = \sum_{i=1}^m P(X_i) (\mu_i - \mu_0)(\mu_i - \mu_0)^T$, where $P(X_i)$ is the percentage of examples in $X_i$, compared to the total number of examples in the dataset. When measuring the distance between the $m$ known classes and the u.u.s class, we set $\mu_0=\mu_u$. Otherwise, if we want to measure the between-class distance within the known classes, we set $\mu_0 = \frac{1}{m}\sum_{i=1}^m\mu_i$. For the \textit{within-class scatter matrix} $S_w$, when measuring the u.u.s class, we simply define $S_w = \Sigma_u$, the covariance matrix of $X_u$. For the known classes, $S_w$ is the weighted sum of all the covariance matrices of the known classes: $S_w = \sum_{i=1}^m P(X_i)\Sigma_i$. To combine both scatter matrices, we use the $J1$ criterion \citep{PatternRecognition}:
\begin{equation}
    J1 = \frac{\text{trace}\{S_w+S_b\}}{\text{trace}\{S_w\}} = 1+\frac{\text{trace}\{S_b\}}{\text{trace}\{S_w\}}
\end{equation}
which increases by making the means of different classes spread out and intra-class variances small.
We evaluate RTSCV on different dataset configurations of varying class means and covariances, each corresponding to a unique $J1$ score. The result is plotted in Figure \ref{fig:ablation} (b) for known classes and Figure \ref{fig:ablation} (c) for unknown classes, which show a positive correlation between model performance and class separability in both cases.

\section{Experimental Evaluation}
\label{sec_experiment}
We evaluate RTSCV on both tabular datasets, where we pair RTSCV with classical classifiers like SVM and simple fully connected neural networks, and computer vision datasets, where we use deep convolutional neural networks, such as ResNet~\citep{ResNet} and DenseNet~\citep{DenseNet}. 

As established by Theorem~\ref{thm:Bayesian}, there is a sweet spot for the sample rate that best balances the classification of $x \in X_s$ by governing the prior $P'(X_s)$. Let sample-training ratio be the ratio of the size of the test set sample and the size of the training set. Figure~\ref{fig:sample_rate} plots RTSCV's performance on COIL-20 and CIFAR10-ImageNet used in Section \ref{sec:exp_tabular} and \ref{sec:exp_cv}, respectively, under varying sample-training ratio. The model performance climbs steeply with the sample size before the sample starts to over-represent the known data, and then decreases slowly. We search for an optimal $c$ using by assessing the misclassification of known data. The optimal sample rate is dataset-dependent and varies from 0.06 to 0.1. Note that the small cardinalities make RTSCV practical for efficiently sampling during model deployment. We also searched for the optimal number $k$, from 2 to 6, of cross-validation for the Letters, Pendigits, COIL-20 and MNIST datasets, and selected $k=3$.

\subsection{Evaluation on Tabular Datasets}
\label{sec:exp_tabular}
We selected the Letter Recognition dataset and Pendigits dataset which contain the hand-writings of 26 English letters and 10 digits, respectively.\footnote{UCI machine learning repository: http://archive.ics.uci.edu/ml} Further, we also down-sample the Columbia University Image Library (COIL-20), which contains grey-scale images of 20 objects~\citep{COIL20}, following the PCA-based technique by~\citep{Geng2020}. We also select the MNIST, which contains 10 digit classes of dimension $28 \times 28$~\citep{mnist}.

On the tabular datasets except for MNIST, we use a standard SVM implementation as the base model. We first report the results of the pre-rectified model, i.e., the performance of the base classifier under the presence of u.u.s without applying RTSCV. To compare with RTSCV, we select three previously proposed methods in the literature of u.u.s discovery, which are (1) EVM \citep{EVM}, (2) 1-vs-Set \citep{OpenSet} and (3) WSVM \citep{WSVM}. In particular, 1-vs-Set and WSVM are both SVM-based algorithms. For MNIST, we use a simple MLP, a four-layer fully connected network, as the base model. Finally, to create u.u.s in the test set of the chosen datasets, we remove certain classes of data from the training set while keeping the test set unchanged.

\subsubsection{Evaluation Metrics}
To be consistent with previous methods, we evaluate the F-measure of our method and other baselines against the \textit{openness}~\citep{OpenSet}.

\textbf{Openness. } The u.u.s may be divided into different classes that span different geometric regions of the feature space. The \textit{openness} metric proposed by \citep{OpenSet} increases with the number of u.u.s classes. A larger openness indicates a larger number of u.u.s classes relative to that of known classes in the test data:
\begin{equation}
    \text{openness} = 1-\sqrt{\frac{2\times |\text{training classes}|}{|\text{test classes}|+|\text{target classes}|}}
\end{equation}

\textbf{F-measure} is the harmonic mean of the precision and recall. In our multi-class scenario, it is obtained by averaging the classwise F-measures, combining classification accuracy of both the known and u.u.s classes.

\begin{table*}[t]
\caption{Experimental results of the RTSCV framework on tabular datasets. Bold and underlined numbers represent the best and second best results, respectively.}
\label{table:tabular-result}
\vskip 0.15in
\begin{center}
\begin{small}
\begin{tabular}{ccccc}
\toprule
\multirow{2}{*}{Dataset}& \multirow{2}{*}{Openness (\%)} & F-measure (\%) \\
&   & Pre-rectified / EVM / 1-vs-Set / WSVM / RTSCV &  \\
\midrule
\multirow{2}{*}{\tabincell{c}{Letter\\ (SVM)}} 
& 14.5 & 69.8 / 89.8 / 72.8 / \underline{91.2} / \textbf{92.1} \\
& 25.5 & 54.2 / 82.8 / 56.1 / \underline{85.7} / \textbf{91.3} \\
\midrule
\multirow{2}{*}{\tabincell{c}{Pendigits\\ (SVM)}} 
& 9.3 & 78.3 / \underline{97.0} / 75.4 / 93.1 / \textbf{97.4} \\
& 18.4 & 69.5 / \underline{92.9} / 62.3 / 88.4 / \textbf{97.2} \\
\midrule
\multirow{2}{*}{\tabincell{c}{COIL-20\\ (SVM)}}
& 9.3 & 88.4 / \textbf{95.7} / 70.2 / 85.6 / \underline{95.1} \\
& 18.4 & 78.7 / \underline{93.2} / 55.7 / 84.5 / \textbf{95.3} \\
\midrule
\multirow{3}{*}{\tabincell{c}{MNIST\\ (MLP)}} 
& 13.4 & 59.6 / \,\, - \,\, / \,\, - \,\, / \,\, - \,\, / \textbf{94.8} \\
& 24.4 & 41.2 / \,\, - \,\, / \,\, - \,\, / \,\, - \,\, / \textbf{96.2} \\
& 42.3 & 21.4 / \,\, - \,\, / \,\, - \,\, / \,\, - \,\, / \textbf{96.8} \\
\bottomrule
\end{tabular}
\end{small}
\end{center}
\vskip -0.1in
\end{table*}

\subsubsection{Results}
We present the experimental results on all four datasets in Table~\ref{table:tabular-result}. The F-measure of the classifier is plotted as a function of openness, which we vary by controlling the number of known classes removed from the training set. RTSCV achieves a consistent performance improvement over the pre-rectified model, closing a performance gap as large as 41\% in some cases. Coupled with the SVM, RTSCV beats previous u.u.s detection methods on 5 out 6 settings while maintaining a small disadvantage to the EVM on COIL-20 with openness 9.3\%. For RTSCV with MLP on MNIST, it performs the best with the largest openess, i.e., 42.3\% where 9 out of 10 classes are selected as u.u.s during test. This further assures the robustness of RTSVC under a disproportionate amount of u.u.s in the test data.

\begin{table}[t!]
\caption{Experimental results of the RTSCV framework on computer vision datasets. Bold and underlined numbers represent the best and second best results, respectively.}
\label{table:cv-result}
\vskip 0.15in
\begin{center}
\begin{small}
\begin{tabular}{ccccc}
\toprule
\multirow{2}{*}{Dataset} & \multirow{2}{*}{u.u.s} & Classification Acc. (\%) & Detection Acc. (\%) & AUROC (\%) \\ \cline{3-5} 
\\[-0.8em]
    &   &  \multicolumn{3}{c}{ Baseline / ODIN / MD / RTSCV} \\
\midrule
\multirow{3}{*}{\tabincell{c}{CIFAR-10\\ (ResNet)}} 
& SVHN & \underline{93.9} / - / \underline{93.9} / \textbf{95.5}  &  85.1 / 91.1 / \underline{95.8} / \textbf{99.7}  &  89.9 / 96.7 / \underline{99.1} / \textbf{99.9} \\
& ImageNet & \underline{93.9} / - / \underline{93.9} / \textbf{95.3} &  85.1 / 86.5 / \underline{99.5} / \textbf{98.8} &  91.0 / 94.0 / \underline{99.5} / \textbf{99.9} \\
& LSUN & \underline{93.9} / - / \underline{93.9} / \textbf{95.5} & 85.3 / 86.7 / \underline{97.7} / \textbf{99.8} & 91.0 / 94.1 / \underline{99.7} / \textbf{100} \\
\midrule
\multirow{3}{*}{\tabincell{c}{CIFAR-100\\ (ResNet)}}
& SVHN & \underline{75.6} / - / 74.8 / \textbf{78.9} & 73.2 / 88.0 / \underline{93.7} / \textbf{99.2} & 79.5 / 93.9 / \underline{98.4} / \textbf{99.9} \\
& ImageNet & \underline{75.6} / - / 74.8 / \textbf{78.9} & 70.8 / 80.1 / \underline{93.3} / \textbf{98.3} & 77.2 / 87.6 / \underline{98.2} / \textbf{99.8} \\
& LSUN & \underline{75.6} / - / 74.8 / \textbf{79.3} & 69.9 / 78.3 / \underline{93.5} / \textbf{99.5} & 75.8 / 85.6 / \underline{98.2} / \textbf{99.9}  \\
\midrule
\multirow{3}{*}{\tabincell{c}{SVHN\\ (ResNet)}}
& CIFAR-10 & \textbf{95.8} / - / \underline{95.7} / 94.5 &  90.0 / 89.4 / \underline{96.9} / \textbf{97.9} &  92.9 / 92.1 / \underline{99.3} / \textbf{99.9} \\
& ImageNet & \textbf{95.8} / - / \underline{95.7} / 94.7 &  90.4 / 89.4 / \textbf{99.1} / \underline{98.4} &  93.5 / 92.0 / \textbf{99.9} / \textbf{99.9} \\
& LSUN &  \textbf{95.8} / - / \underline{95.7} / 95.0  &  89.0 / 87.2 / \underline{99.5} / \textbf{99.8} &  91.6 / 89.4 / \underline{99.9} / \textbf{100} \\
\midrule
\multirow{3}{*}{\tabincell{c}{CIFAR-10\\ (DenseNet)}}
& SVHN & \underline{92.9} / - / 91.7 / \textbf{94.8} & 83.2 / 91.4 / \underline{93.9} / \textbf{98.1} & 89.9 / 95.5 / \underline{98.1} / \textbf{99.9} \\
& ImageNet & \underline{92.9} / - / 91.7 / \textbf{95.0} & 88.5 / 93.9 / \underline{95.0} / \textbf{95.2} & 94.1 / 98.5 / \underline{98.8} / \textbf{99.9} \\
& LSUN & \underline{92.9} / - / 91.7 / \textbf{94.9} & 90.3 / 95.7 / \underline{96.3} / \textbf{98.7} & 95.4 / 99.2 / \underline{99.3} / \textbf{99.9} \\
\midrule
\multirow{3}{*}{\tabincell{c}{CIFAR-100\\ (DenseNet)}}
& SVHN & \underline{72.3} / - / 68.2 / \textbf{72.8}  & 75.6 / 86.6 / \underline{91.5} / \textbf{98.8} & 82.7 / 93.8 / \underline{97.2} / \textbf{99.9}  \\
& ImageNet & \underline{72.3} / - / 68.2 / \textbf{73.4} & 65.7 / 77.0 / \underline{92.2} / \textbf{95.8} & 71.7 / 85.2 / \underline{97.4} / \textbf{99.8} \\
& LSUN & \underline{72.3} / - / 68.2 / \textbf{73.5} & 64.9 / 77.1 / \underline{93.9} / \textbf{98.9} &  70.8 / 85.5 / \underline{98.0} / \textbf{99.9} \\
\midrule
\multirow{3}{*}{\tabincell{c}{SVHN\\ (DenseNet)}} 
& CIFAR-10 & \underline{95.3} / - / 95.2 / \textbf{96.6}  & 86.6 / 85.8 / \underline{95.9} / \textbf{96.9} & 91.9 / 91.4 / \underline{98.9} / \textbf{99.9} \\
& ImageNet & \underline{95.3} / - / 95.2 / \textbf{96.6} & 90.2 / 90.4 / \textbf{98.9} / \underline{97.4} & 94.8 / 95.1 / \textbf{99.9} / \textbf{99.9} \\
& LSUN & \underline{95.3} / - / 95.2 / \textbf{96.8} & 89.1 / 89.2 / \underline{99.3} / \textbf{99.9} & 94.1 / 94.5 / \underline{99.9} / \textbf{100} \\
\bottomrule
\end{tabular}
\end{small}
\end{center}
\vskip -0.1in
\end{table}

\subsection{Evaluation on Computer Vision Datasets}
\label{sec:exp_cv}
We also evaluate RTSCV on several more challenging pattern recognition tasks, such as computer vision datasets. Specifically, we select the CIFAR-10 and CIFAR-100, both containing colored object images \citep{cifar10}, as well as the Street View House Numbers (SVHN) from Google Street View project~\citep{svhn}. All image dimensions of these datasets are $32 \times 32$.

For each of the computer vision datasets, we separately test our RTSCV framework with ResNet \citep{ResNet} and DenseNet \citep{DenseNet}, two network architectures that have achieved good performance on large benchmark datasets. (See the supplementary material for details on the training configuration). For comparison, we include the results of the baseline method by \citep{2017OOD-baseline}, the ODIN \citep{ODIN-OOD}, and the Mahalanobis method (MD) \citep{MD-OOD}, three previous approaches on detecting OOD samples (u.u.s) for neural networks. Different from Section \ref{sec:exp_tabular}, here we create u.u.s in the test set by introducing additional computer vision datasets to the target space. Specifically, we use the resized version of Tiny-ImageNet \citep{image-net} and LSUN \citep{lsun} following the techniques in \citep{ODIN-OOD,MD-OOD}. We present a summary of the known datasets and u.u.s datasets we used in Table~\ref{table:cv-result}.
 
\subsubsection{Evaluation Metrics}
We adopt the following metrics to evaluate performance on both known classes classification and u.u.s detection.  

\textbf{Classification accuracy} is the accuracy of the known class classification, i.e., the total number of correct predictions on the known class labels divided by the total number of known class test samples.

\textbf{Detection accuracy} is the number of u.u.s in the test data that are correctly detected by the model divided by the total number of the u.u.s.

\textbf{AUROC} depicts the relationship between true positive rate (TPR) and false positive rate (FPR)~\citep{auroc}. A higher AUROC indicates a higher probability for a positive instance to rank higher than a negative one.

\subsubsection{Results}
We display the experimental results in Table~\ref{table:cv-result}. RTSCV coupled with ResNet or DenseNet achieves the overall best performance with respect to all of the evaluation metrics. In particular, it has a significant improvement on the u.u.s detection accuracy while maintaining a high classification accuracy for the known classes. This suggests that RTSCV is the most effective in identifying u.u.s, without degrading the original model, even on more complex datasets with more complex classification models.

\section{Conclusion and Future Work}
With the goal of reducing deployment errors and model bias due to deficient training data, RTSCV is a proposal of an algorithmic framework that adds the flexibility for a base classification models to rectify a trained model at deployment. We provide a rigorous theoretical analysis of correctness and performance guarantees of the process of minimizing its structural mismatch with a target space, based on objectives that most modern classifers aim to optimize. RTSCV exhibits consistent performance improvements over both the pre-rectified model and previously proposed approaches that share the same goals on 7 benchmark datasets. Moreover, it does not assume the presence of an oracle, as in the case of active learning.

Our ongoing work focus on improvements of the RTSCV, especially due to concerns with the computational cost of cross-validation. We are developing alternatives to cross-validation that could work equally well, while reducing the computational cost dramatically. 
For instance, we are in the process of investigating semi-supervised clustering \citep{semi-supervised-survey,semi-supervised-clustering}. In the supplementary materials, we discuss our initial efforts in this direction. Our preliminary results suggest that such an approach works equally well as RTSCV when the u.u.s consist of only one cluster and have a relatively small covariance. Nevertheless, when the u.u.s form multiple clusters or the clusters have high variance, the performance of this alternative approach drops significantly, compared to that of RTSCV.

\section{Acknowledgements} 

This research was partially supported by a National Natural Science Foundation of China (NSFC) grant \#61850410536. Abrahao developed part of this research while affiliated with Microsoft Research AI, Redmond.

\bibliographystyle{unsrtnat}
\bibliography{references} 

\begin{thebibliography}{33}
\providecommand{\natexlab}[1]{#1}
\providecommand{\url}[1]{\texttt{#1}}
\expandafter\ifx\csname urlstyle\endcsname\relax
  \providecommand{\doi}[1]{doi: #1}\else
  \providecommand{\doi}{doi: \begingroup \urlstyle{rm}\Url}\fi

\bibitem[Lakkaraju et~al.(2017)Lakkaraju, Kamar, Caruana, and
  Horvitz]{Lakkaraju2017}
Himabindu Lakkaraju, Ece Kamar, Rich Caruana, and Eric Horvitz.
\newblock Identifying unknown unknowns in the open world: Representations and
  policies for guided exploration.
\newblock In \emph{Proc. of the Thirty-First AAAI Conference on Artificial
  Intelligence}, 2017.

\bibitem[Hendrycks and Gimpel(2017)]{2017OOD-baseline}
Dan Hendrycks and Kevin Gimpel.
\newblock A baseline for detecting misclassified and out-of-distribution
  examples in neural networks.
\newblock \emph{ICLR}, 2017.

\bibitem[Liang et~al.(2018)Liang, Li, and Srikant]{ODIN-OOD}
Shiyu Liang, Y.~Li, and R.~Srikant.
\newblock Enhancing the reliability of out-of-distribution image detection in
  neural networks.
\newblock \emph{ICLR}, 2018.

\bibitem[Lee et~al.(2018)Lee, Lee, Lee, and Shin]{MD-OOD}
Kimin Lee, Kibok Lee, Honglak Lee, and Jinwoo Shin.
\newblock A simple unified framework for detecting out-of-distribution samples
  and adversarial attacks.
\newblock In \emph{Advances in Neural Information Processing Systems},
  volume~31, pages 7167--7177, 2018.

\bibitem[Liu et~al.(2020)Liu, Wang, Owens, and Li]{energy-OOD}
Weitang Liu, Xiaoyun Wang, John Owens, and Yixuan Li.
\newblock Energy-based out-of-distribution detection.
\newblock \emph{Advances in Neural Information Processing Systems (NeurIPS)},
  2020.

\bibitem[{Scheirer} et~al.(2013){Scheirer}, {de Rezende Rocha}, {Sapkota}, and
  {Boult}]{OpenSet}
W.~J. {Scheirer}, A.~{de Rezende Rocha}, A.~{Sapkota}, and T.~E. {Boult}.
\newblock Toward open set recognition.
\newblock \emph{IEEE Transactions on Pattern Analysis and Machine
  Intelligence}, 2013.

\bibitem[{Scheirer} et~al.(2014){Scheirer}, {Jain}, and {Boult}]{WSVM}
W.~J. {Scheirer}, L.~P. {Jain}, and T.~E. {Boult}.
\newblock Probability models for open set recognition.
\newblock \emph{IEEE Transactions on Pattern Analysis and Machine
  Intelligence}, 2014.

\bibitem[J{\'u}nior et~al.(2016)J{\'u}nior, de~Souza, de~Oliveira~Werneck,
  Stein, Pazinato, de~Almeida, Penatti, da~Silva~Torres, and
  Rocha]{Jnior2016NearestKNN}
Pedro Ribeiro~Mendes J{\'u}nior, Roberto~Medeiros de~Souza, Rafael
  de~Oliveira~Werneck, Bernardo~V. Stein, Daniel~V. Pazinato, Waldir~R.
  de~Almeida, Ot{\'a}vio Augusto~Bizetto Penatti, Ricardo da~Silva~Torres, and
  Anderson Rocha.
\newblock Nearest neighbors distance ratio open-set classifier.
\newblock \emph{Machine Learning}, 2016.

\bibitem[Bendale and Boult()]{OpenMax}
Abhijit Bendale and Terrance~E. Boult.
\newblock Towards open set deep networks.
\newblock In \emph{Proceedings of the IEEE Conference on Computer Vision and
  Pattern Recognition (CVPR-2016)}.

\bibitem[Vandenhof and Law(2019)]{contradict-the-machine}
Colin Vandenhof and Edith Law.
\newblock Contradict the machine: A hybrid approach to identifying unknown
  unknowns.
\newblock In \emph{AAMAS}, 2019.

\bibitem[Simard et~al.(2017)Simard, Amershi, Chickering, Pelton, Ghorashi,
  Meek, Ramos, Suh, Verwey, Wang, and Wernsing]{MachineTeaching}
Patrice~Y. Simard, Saleema Amershi, David~Maxwell Chickering, Alicia~Edelman
  Pelton, Soroush Ghorashi, Christopher Meek, Gonzalo Ramos, Jina Suh, Johan
  Verwey, Mo~Wang, and John~Robert Wernsing.
\newblock Machine teaching: A new paradigm for building machine learning
  systems.
\newblock \emph{ArXiv}, abs/1707.06742, 2017.

\bibitem[Jain et~al.()Jain, Scheirer, and Boult]{PISVM}
Lalit~P. Jain, Walter~J. Scheirer, and Terrance~E. Boult.
\newblock Multi-class open set recognition using probability of inclusion.
\newblock In \emph{ECCV 2014}.

\bibitem[Lee et~al.(2020)Lee, Yu, and Yu]{LH-OOD}
Dongha Lee, Sehun Yu, and Hwanjo Yu.
\newblock Multi-class data description for out-of-distribution detection.
\newblock In \emph{Proceedings of the 26th ACM SIGKDD International Conference
  on Knowledge Discovery \& Data Mining}, KDD '20, New York, NY, USA, 2020.

\bibitem[{Yoshihashi} et~al.(){Yoshihashi}, {Shao}, {Kawakami}, {You}, {Iida},
  and {Naemura}]{CROSR}
R.~{Yoshihashi}, W.~{Shao}, R.~{Kawakami}, S.~{You}, M.~{Iida}, and
  T.~{Naemura}.
\newblock Classification-reconstruction learning for open-set recognition.
\newblock In \emph{2019 IEEE/CVF Conference on Computer Vision and Pattern
  Recognition (CVPR)}.

\bibitem[Oza and Patel()]{C2AE}
Poojan Oza and Vishal~M. Patel.
\newblock {C2AE:} class conditioned auto-encoder for open-set recognition.
\newblock In \emph{{IEEE} Conference on Computer Vision and Pattern
  Recognition, {CVPR} 2019}.
\newblock \doi{10.1109/CVPR.2019.00241}.

\bibitem[{Geng} et~al.(2020){Geng}, {Huang}, and {Chen}]{Geng2020Survey}
C.~{Geng}, S.~{Huang}, and S.~{Chen}.
\newblock Recent advances in open set recognition: A survey.
\newblock \emph{IEEE Transactions on Pattern Analysis and Machine
  Intelligence}, 2020.

\bibitem[Boult et~al.(2019)Boult, Cruz, Dhamija, Günther, Henrydoss, and
  Scheirer]{2019Survey}
Terrance Boult, S.~Cruz, Akshay Dhamija, Manuel Günther, James Henrydoss, and
  W.J. Scheirer.
\newblock Learning and the unknown: Surveying steps toward open world
  recognition.
\newblock In \emph{Proceedings of the 33th AAAI Conference on Artificial
  Intelligence}, 2019.
\newblock \doi{10.1609/aaai.v33i01.33019801}.

\bibitem[{Rudd} et~al.(2018){Rudd}, {Jain}, {Scheirer}, and {Boult}]{EVM}
E.~M. {Rudd}, L.~P. {Jain}, W.~J. {Scheirer}, and T.~E. {Boult}.
\newblock The extreme value machine.
\newblock \emph{IEEE Transactions on Pattern Analysis and Machine
  Intelligence}, 2018.

\bibitem[Brodley and Friedl(1999)]{Brodley1999}
Carla~E. Brodley and Mark~A. Friedl.
\newblock Identifying mislabeled training data.
\newblock \emph{Journal of Artificial Intelligence Research}, 11\penalty0
  (1):\penalty0 131--167, July 1999.
\newblock ISSN 1076-9757.

\bibitem[Murty and Devi(2011)]{naiveBayes}
M.~Murty and V.~Devi.
\newblock \emph{Pattern recognition. An algorithmic approach}.
\newblock 2011.
\newblock \doi{10.1007/978-0-85729-495-1}.

\bibitem[Theodoridis and Koutroumbas(2008)]{PatternRecognition}
Sergios Theodoridis and Konstantinos Koutroumbas.
\newblock \emph{Pattern Recognition, Fourth Edition}.
\newblock Academic Press, Inc., USA, 4th edition, 2008.
\newblock ISBN 1597492728.

\bibitem[{He} et~al.(2016){He}, {Zhang}, {Ren}, and {Sun}]{ResNet}
K.~{He}, X.~{Zhang}, S.~{Ren}, and J.~{Sun}.
\newblock Deep residual learning for image recognition.
\newblock In \emph{2016 IEEE Conference on Computer Vision and Pattern
  Recognition (CVPR)}, pages 770--778, 2016.
\newblock \doi{10.1109/CVPR.2016.90}.

\bibitem[{Huang} et~al.(2017){Huang}, {Liu}, {Van Der Maaten}, and
  {Weinberger}]{DenseNet}
G.~{Huang}, Z.~{Liu}, L.~{Van Der Maaten}, and K.~Q. {Weinberger}.
\newblock Densely connected convolutional networks.
\newblock In \emph{2017 IEEE Conference on Computer Vision and Pattern
  Recognition (CVPR)}, 2017.
\newblock \doi{10.1109/CVPR.2017.243}.

\bibitem[Nene et~al.(1996)Nene, Nayar, and Murase]{COIL20}
Sameer~A. Nene, Shree~K. Nayar, and Hiroshi Murase.
\newblock Columbia object image library (coil-20).
\newblock Technical report, 1996.

\bibitem[Geng and Chen(2020)]{Geng2020}
Chuanxing Geng and Songcan Chen.
\newblock Collective decision for open set recognition.
\newblock \emph{IEEE Transactions on Knowledge and Data Engineering}, 2020.
\newblock \doi{10.1109/tkde.2020.2978199}.

\bibitem[LeCun et~al.(2010)LeCun, Cortes, and Burges]{mnist}
Yann LeCun, Corinna Cortes, and CJ~Burges.
\newblock Mnist handwritten digit database.
\newblock \emph{ATT Labs [Online]. Available:
  http://yann.lecun.com/exdb/mnist}, 2, 2010.

\bibitem[Krizhevsky(2009)]{cifar10}
Alex Krizhevsky.
\newblock Learning multiple layers of features from tiny images.
\newblock Technical report, 2009.

\bibitem[Netzer et~al.(2011)Netzer, Wang, Coates, Bissacco, Wu, and Ng]{svhn}
Yuval Netzer, Tao Wang, Adam Coates, Alessandro Bissacco, Bo~Wu, and Andrew Ng.
\newblock Reading digits in natural images with unsupervised feature learning.
\newblock \emph{NIPS}, 2011.

\bibitem[{Deng} et~al.(2009){Deng}, {Dong}, {Socher}, {Li}, {Kai Li}, and {Li
  Fei-Fei}]{image-net}
J.~{Deng}, W.~{Dong}, R.~{Socher}, L.~{Li}, {Kai Li}, and {Li Fei-Fei}.
\newblock Imagenet: A large-scale hierarchical image database.
\newblock In \emph{2009 IEEE Conference on Computer Vision and Pattern
  Recognition}, 2009.
\newblock \doi{10.1109/CVPR.2009.5206848}.

\bibitem[Yu et~al.(2015)Yu, Zhang, Song, Seff, and Xiao]{lsun}
Fisher Yu, Yinda Zhang, Shuran Song, Ari Seff, and Jianxiong Xiao.
\newblock Lsun: Construction of a large-scale image dataset using deep learning
  with humans in the loop.
\newblock 2015.

\bibitem[Davis and Goadrich(2006)]{auroc}
Jesse Davis and Mark Goadrich.
\newblock The relationship between precision-recall and roc curves.
\newblock In \emph{Proceedings of the 23rd International Conference on Machine
  Learning, ACM}, 2006.

\bibitem[Zhu(2008)]{semi-supervised-survey}
Xiaojin Zhu.
\newblock Semi-supervised learning literature survey.
\newblock \emph{Computer Science, University of Wisconsin-Madison}, 07 2008.

\bibitem[Basu et~al.(2002)Basu, Banerjee, and
  Mooney]{semi-supervised-clustering}
Sugato Basu, Arindam Banerjee, and R.~Mooney.
\newblock Semi-supervised clustering by seeding.
\newblock In \emph{Proceedings of 19th International Conference on Machine
  Learning (ICML-2002)}, 2002.

\end{thebibliography}

\setcounter{section}{0}
\newpage
\section{Proofs of Theorems}
\subsection{Theorem 3.1}
\begin{proof}
Let's first prove the case of $x \sim X_k$. We know that
$$\E_{x\sim X_k}(L_k(x)) = \underset{\mathbb{R}^{d}}{\int}L_k^{2}(x)dx = \mathcal{N}(\mu_{k},\mu_{k},2\Sigma_{k})$$
$$\E_{x\sim X_k}(L_i(x)) = \underset{\mathbb{R}^{d}}{\int}L_i(x)L_{k}(x)dx = \mathcal{N}(\mu_{i},\mu_{k},\Sigma_{k}+\Sigma_{i})$$
Expanding the two terms, we have
$$\E_{x\sim X_k}(L_k(x)) = \frac{1}{(2\pi)^{\frac{d}{2}}\sqrt{|2\Sigma_k|}}\exp(-\frac{1}{2}(\mu_k - \mu_k)^{T}(2\Sigma_k)^{-1}(\mu_k - \mu_k)) = \frac{1}{(2\pi)^{\frac{d}{2}}\sqrt{|2\Sigma_k|}}$$
$$\E_{x\sim X_k}(L_i(x)) = \frac{1}{(2\pi)^{\frac{d}{2}}\sqrt{|\Sigma_k+\Sigma_{i}|}}\exp(-\frac{1}{2}(\mu_i - \mu_k)^{T}(\Sigma_k+\Sigma_i)^{-1}(\mu_i - \mu_k))$$
Recall that $\Sigma_i = \sigma_i^{2}I,\Sigma_u = \sigma_u^{2}I$. Let $\mathcal{N}(\mu_{k},\mu_{k},2\Sigma_{k}) >  \mathcal{N}(\mu_{i},\mu_{k},\Sigma_{k}+\Sigma_{i}) $ and we have
$$\|\mu_{k}-\mu_{i}\|^{2} \geq d \cdot  (\sigma_{i}^{2}+\sigma_{k}^{2}) \cdot \ln \bigg(\frac{2\sigma_{k}^{2}}{\sigma_{k}^{2}+\sigma_{i}^{2}}\bigg)$$
If for all other classes other than class $k$, the above conditions hold, then
$$\E_{x\sim X_k}(L_k(x)) = P(X_u)\E_{x\sim X_k}(L_k(x)) + \sum_{i=1}^m P(X_i) \E_{x\sim X_k}(L_k(x)) >  \E_{x\sim X_k}(L_s(x))$$
since $P(X_u)+\sum_{i=1}^m P(X_i)=1$. Similarly, for the second case of $x \sim X_u$, we can compute that $\E_{x\sim X_u}(L_k(x)) = \mathcal{N}(\mu_k,\mu_u,\Sigma_u+\Sigma_k)$. Notice that 
$$\E_{x\sim X_u}(L_s(x)) > P(X_u)\mathcal{N}(\mu_u,\mu_u,2\Sigma_u) + P(X_k)\mathcal{N}(\mu_k,\mu_u,\Sigma_u+\Sigma_k)$$
Hence, to obtain a sufficient condition of $\mathbb{E}_{x\sim X_u}(L_s(x)) > \mathbb{E}_{x\sim X_u}(L_k(x))$, we simply need to require $P(X_u)\mathcal{N}(\mu_u,\mu_u,2\Sigma_u) > (1-P(X_k)) \mathcal{N}(\mu_k,\mu_u,\Sigma_u+\Sigma_k)$, which yields
$$\frac{\|\mu_{k}-\mu_{u}\|_2^{2}}{\sigma_{k}^{2}+\sigma_{u}^{2}} \ge 2\ln \bigg(\frac{1-P(X_k)}{P(X_u)}\bigg)+ d \cdot \ln \bigg(\frac{2\sigma_{u}^{2}}{\sigma_{k}^{2}+\sigma_{u}^{2}}\bigg)$$
\end{proof}

\subsection{Theorem 3.2}
\begin{proof}
For $x \sim X_k$, to show that $\E_{x\sim X_k}(p(X_k|x)) > \E_{x\sim X_k}(p(X_s|x))$, by Bayes' theorem, we just need to show:
$$\E_{x\sim X_k}(L_{k}(x)P^{\prime}(X_k)) > \E_{x\sim X_k}(L_{s}(x)P^{\prime}(X_s))$$
From the proof of Theorem 3.1, we have:
$$\E_{x\sim X_k}(L_{k}(x)P^{\prime}(X_k)) = P^{\prime}(X_k)\mathcal{N}(\mu_{k},\mu_{k},2\Sigma_{k})$$
$$\E_{x\sim X_k}(L_{i}(x)P^{\prime}(X_s)) = P^{\prime}(X_s)\mathcal{N}(\mu_{i},\mu_{k},\Sigma_{k}+\Sigma_{i})$$
Let $P'(X_k)\mathcal{N}(\mu_{k},\mu_{k},2\Sigma_{k}) > P'(X_s)\mathcal{N}(\mu_{i},\mu_{k},\Sigma_{k}+\Sigma_{i})$, we obtain:
$$\frac{\|\mu_{k}-\mu_{i}\|_2^{2}}{\sigma_{k}^{2}+\sigma_{i}^{2}} \geq 2\ln \bigg(\frac{P^{\prime}(X_s)}{P^{\prime}(X_k)}\bigg)+ d \cdot \ln \bigg(\frac{2\sigma_{k}^{2}}{\sigma_{k}^{2}+\sigma_{i}^{2}}\bigg)$$

For $x \sim X_u$, We just need to show that $\E_{x\sim X_u}(L_{s}(x)P^{\prime}(X_s)) > \E_{x\sim X_u}(L_{k}(x)P^{\prime}(X_k))$. 
From the proof of Theorem 3.1, we know that:
$$\E_{x\sim X_u}(L_{s}(x)P^{\prime}(X_s)) > P^{\prime}(X_s)P(X_u)\mathcal{N}(\mu_u,\mu_u,2\Sigma_u) + P^{\prime}(X_s)P(X_k)\mathcal{N}(\mu_k,\mu_u,\Sigma_u+\Sigma_k)$$
$$\E_{x\sim X_u}(L_{k}(x)P^{\prime}(X_k)) = P^{\prime}(X_k)\mathcal{N}(\mu_k,\mu_u,\Sigma_u+\Sigma_k)$$
Letting $P^{\prime}(X_s)P(X_u)\mathcal{N}(\mu_u,\mu_u,2\Sigma_u) > [P^{\prime}(X_k)- P^{\prime}(X_s)P(X_k)] \mathcal{N}(\mu_k,\mu_u,\Sigma_u+\Sigma_k)$ gives the desired sufficient condition:
$$\frac{\|\mu_{k}-\mu_{u}\|_2^{2}}{\sigma_{k}^{2}+\sigma_{u}^{2}} \geq 2\ln \bigg(\frac{P^{\prime}(X_k)- P^{\prime}(X_s)P(X_k)}{P^{\prime}(X_s)P(X_u)}\bigg)
+ d \cdot \ln \bigg(\frac{2\sigma_{u}^{2}}{\sigma_{k}^{2}+\sigma_{u}^{2}}\bigg)$$
\end{proof}

\subsection{Theorem 3.3}
\begin{proof}
For $x \sim X_k$, we know from statistical theory that $D_M^{2}(x, X_k)$ has a $\chi_{d}^{2}$ distribution with $d$ degrees of freedom. So we have
$$\E_{x\sim X_k}(D_M^{2}(x, X_k)) = \E_{x\sim \chi_{d}^{2}}(x) = d$$
Let's now investigate the distribution of $D_M^{2}(x, X_s)$. As $\Sigma_{s}$ is real symmetric and diagonalizable, it has an orthogonal decomposition:
$$\Sigma_{s} = U \Lambda U^{-1} = U \Lambda U^{T} = \sum_{j=1}^{d} \lambda_{j} u_{j} u_{j}^{T}$$
$$\Sigma_{s}^{-1} = U \Lambda^{-1} U^{-1} = U \Lambda^{-1} U^{T} = \sum_{j=1}^{d} \lambda_{j}^{-1} u_{j} u_{j}^{T}$$
where $\{\lambda_{j}\}_{j = 1}^{d}$ are the eigenvalues for $\Lambda$ and $\{u_{j}\}_{j = 1}^{d}$ are the corresponding eigenvectors. Plugging the decomposition into the formula of $D_M^{2}(x, X_s)$, we get
$$D_M^{2}(x, X_s) = (x-\mu_{s})^{T}\Sigma_{s}^{-1}(x - \mu_{s})= \sum_{j = 1}^{d} [\lambda_{j}^{-\frac{1}{2}}u_{j}^{T}(x - \mu_{s})]^{2}=\sum_{j = 1}^{d} Y_{j}^{2}$$
Since $x \sim \mathcal{N}(\mu_k,\Sigma_k)$, $Y_{j} = \lambda_{j}^{-\frac{1}{2}}u_{j}^{T}(x - \mu_{s})$ is an affine transformation of a multivariate Gaussian distribution, $Y_{k}$ has a univariate normal distribution:
$$\E(Y_{j}) = \lambda_{j}^{-\frac{1}{2}}u_{j}^{T} (\E(X) - \mu_{s}) = \lambda_{j}^{-\frac{1}{2}}u_{j}^{T} (\mu_{k} - \mu_{s})$$
$$Var(Y_{j}) = \lambda_{j}^{-\frac{1}{2}}u_{j}^T \Sigma_{k} \lambda_{j}^{-\frac{1}{2}}u_{j} = \frac{\sigma_{k}^{2}}{\lambda_{j}}$$
Therefore, we can infer that
$$\E(Y_{j}^{2}) =  \sigma_{k}^{2}\lambda_{j}^{-1} + [\lambda_{j}^{-\frac{1}{2}}u_{j}^{T} (\mu_{k} - \mu_{s})]^{2}$$
$$\E_{x\sim X_k}(D_M^{2}(x, X_s)) = \sum_{j=1}^d \E(Y_{j}^{2}) = \sigma_{k}^{2}\sum_{j=1}^{d}\lambda_{j}^{-1} + \sum_{j=1}^{d}\frac{[u_{j}^{T}(\mu_{k}-\mu_{s})]^{2}}{\lambda_{j}}$$
Here, since we assume $\Sigma_s = \sigma_{s}^{2}I$, $\{u_{k}\}_{k=1}^{d}$ is the canonical basis of $\mathbb{R}^{d}$. So the formula can be further simplified as:
$$\E_{x\sim X_k}(D_M^{2}(x, X_s)) = d\cdot \frac{\sigma_{k}^{2}}{\sigma_{s}^{2}} + \frac{\|\mu_{k}-\mu_{s}\|^{2}}{\sigma_{s}^{2}}$$
Let $\E_{x\sim X_k}(D_M^{2}(x, X_k)) < \E_{x\sim X_k}(D_M^{2}(\boldsymbol{x}, X_s))$, we obtain the sufficient condition:
$$\|\mu_{k} - \mu_{s}\|^{2} > d \cdot \sigma_{s}^{2} \cdot (1 - \frac{\sigma_{k}^{2}}{\sigma_{s}^{2}})$$
Similarly, for the case of $x \sim X_u$, following the above results we have:
$$\E_{x\sim X_u}(D_M^{2}(x, X_s)) = d\cdot \frac{\sigma_{u}^{2}}{\sigma_{s}^{2}} + \frac{\|\mu_{u}-\mu_{s}\|^{2}}{\sigma_{s}^{2}}$$
$$\E_{x\sim X_u}(D_M^{2}(x, X_k)) = d\cdot \frac{\sigma_{u}^{2}}{\sigma_{k}^{2}} + \frac{\|\mu_{u}-\mu_{k}\|^{2}}{\sigma_{k}^{2}}$$
Let $\E_{x\sim X_u}(D_M^{2}(x, X_s)) < \E_{x\sim X_u}(D_M^{2}(x, X_k))$ and we have:
$$\frac{\|\mu_{u}-\mu_{k}\|^{2}}{\sigma_{k}^{2}} - \frac{\|\mu_{u}-\mu_{s}\|^{2}}{\sigma_{s}^{2}} > d \cdot \Big(\frac{\sigma_{u}^{2}}{\sigma_{s}^{2}} - \frac{\sigma_{u}^{2}}{\sigma_{k}^{2}}\Big)$$
\end{proof}

\newpage
\section{Training Configuration for Computer Vision Datasets}
\begin{table}[h]
\caption{ResNet}
\begin{center}
\begin{tabular}{cc}
\textbf{Parameter} & \textbf{Value} \\ \hline
\\[-0.8em]
optimizer & SGD with Nesterov Momentum \\
momentum & 0.9 \\
learning rate & 5e-4 \\
epochs & 100 \\
learning rate scheduler &  learning rate decreases by 50\% after every 20 epochs\\
cross-validation folds & 3 \\
number of layers & 34
\end{tabular}
\end{center}
\end{table}

\begin{table}[h]
\caption{DenseNet}
\begin{center}
\begin{tabular}{cc}
\textbf{Parameter} & \textbf{Value} \\ \hline
\\[-0.8em]
optimizer & SGD with Nesterov Momentum \\
momentum & 0.9 \\
learning rate & 5e-4 \\
epochs & 100 \\
learning rate scheduler &  learning rate decreases by 50\% after every 20 epochs\\
cross-validation folds & 3 \\
number of layers & 100
\end{tabular}
\end{center}
\end{table}

\newpage
\section{Discussion on Semi-Supervised Clustering}
\begin{figure}[h]
\centering
\includegraphics[width=0.45\textwidth]{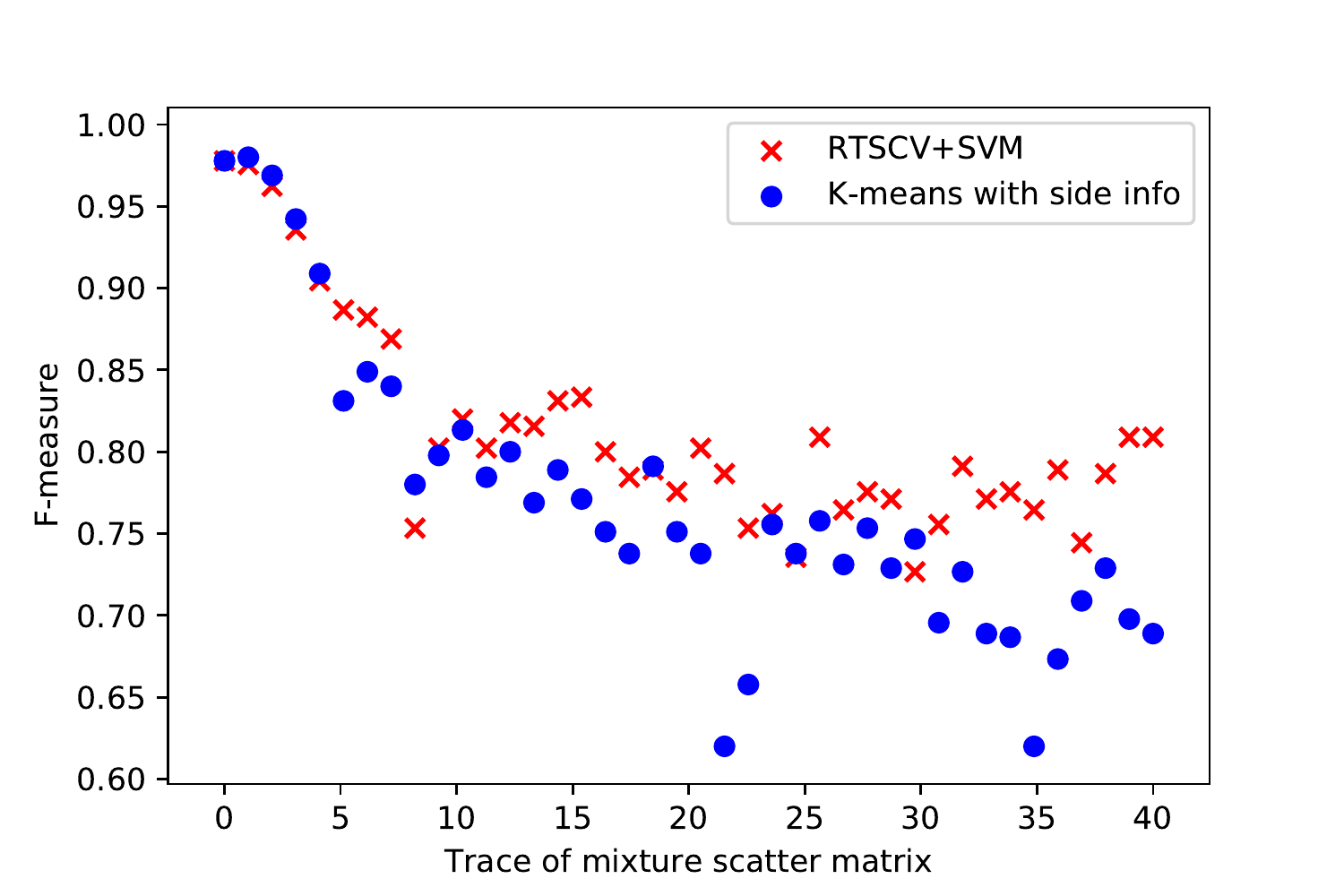}
\quad
\includegraphics[width=0.45\textwidth]{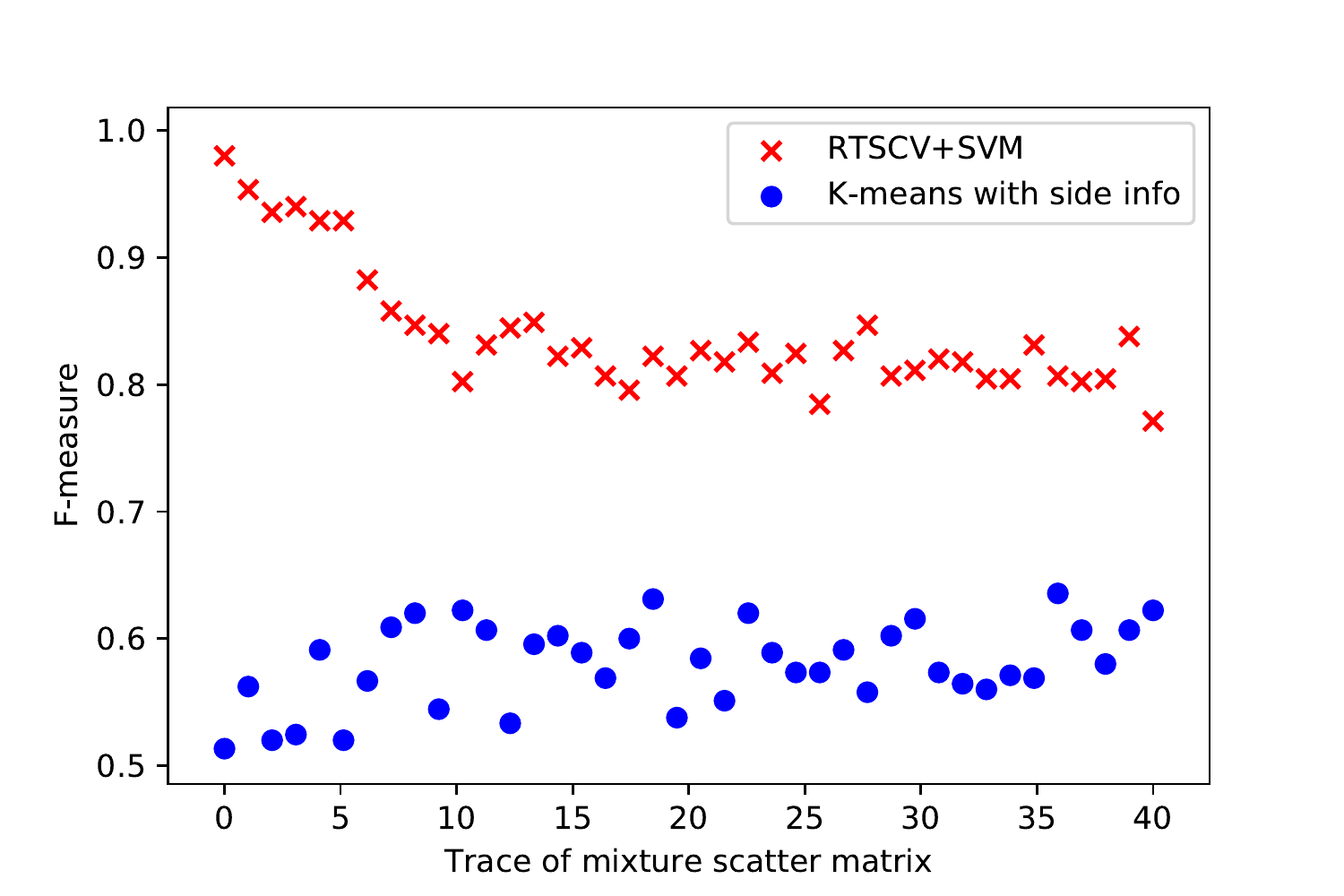}
\qquad
\includegraphics[width=0.45\textwidth]{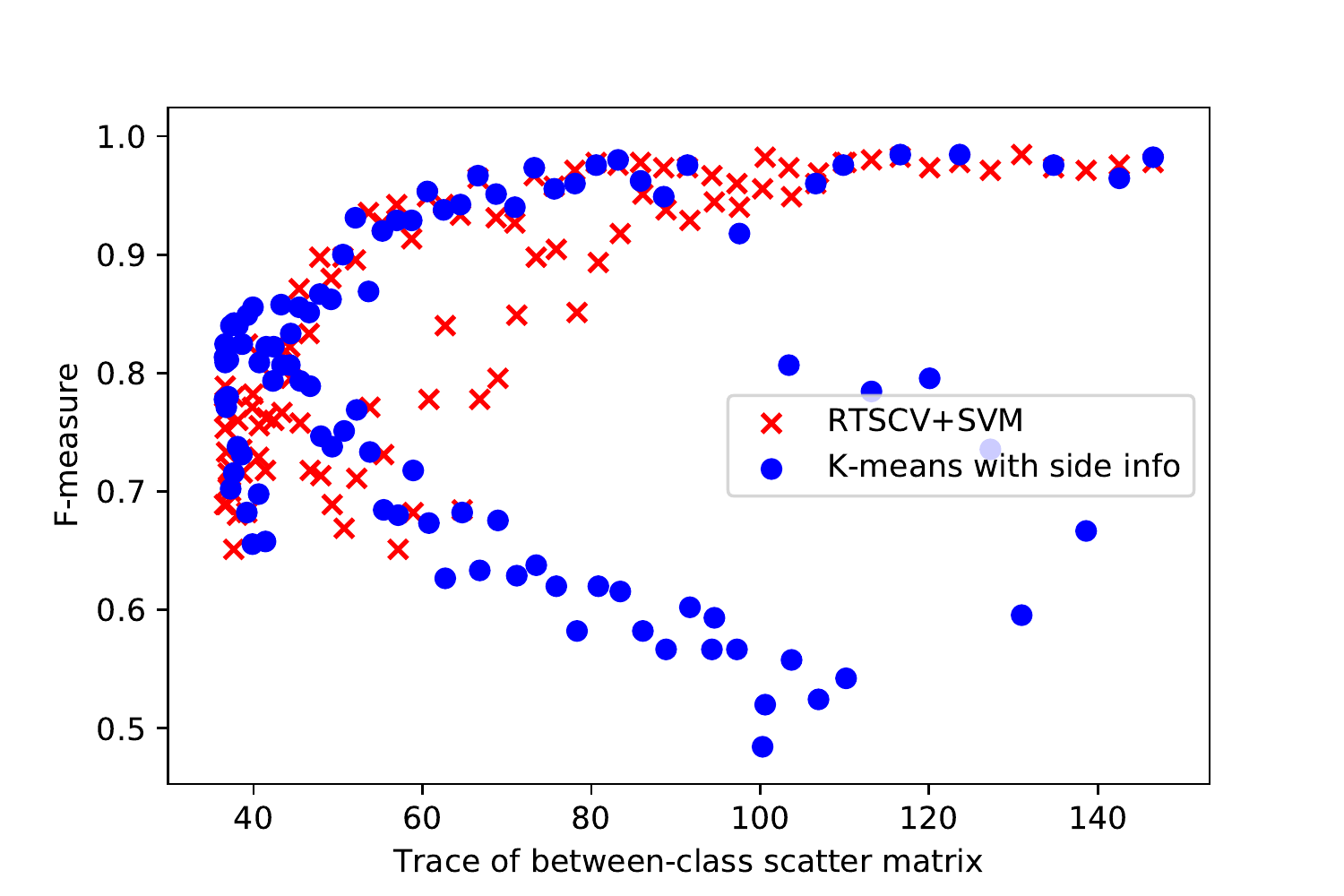}
\caption{Comparison between Clustering with Side Information (CSI) and our RTSCV methods under different synthetic dataset settings. Top: There is one u.u. cluster for the left plot and two u.u. clusters for the right plot. For both plots the u.u. class is located far away from the 10 known classes and only the covariance of the u.u. class is altered across different trials. }
\label{fig:csi}
\end{figure}

We also believe that it is worth discussing the possibility of resorting to semi-supervised clustering as an alternative to cross-validation during the process of re-classifying sample class $X_s$, given its increasing popularity and the great potential of being more computationally economical. Given a small amount of labeled data, semi-supervised clustering performs ordinary clustering tasks under the constraints of \textit{must-links} (two points must be in the same cluster) and \textit{cannot-links} (two points cannot be in the same cluster), provided by the labeled data \citep{semi-supervised-survey}. In our scenario, the objective of re-classifying $X_s$ can be viewed equivalent to dividing $X_s$ into several clusters, one of which corresponds to either a known or u.u. class, with the assistance of the labeled data from the entire training set. This is also called \textit{clustering with side information (CSI)} in the literature \citep{semi-supervised-survey}. 

To test this alternative, we adopt a novel but simple method called Seeded-KMeans \citep{semi-supervised-clustering}. Specifically, given $M$ known classes $X_1,X_2,\ldots,X_M$ in the training set, we run an $(M+1)$-Means clustering algorithm on sample set $X_s$, with the initial centers of each cluster set to the mean feature vectors of $X_1, X_2,\ldots,X_M$ and $X_s$, respectively. After the clustering converges, we assign the label of each cluster of $X_s$ according to the class membership of the initial seeding of the corresponding center. In other words, a cluster initially seeded by the mean of some known class $X_k$ will be labeled as $X_k$, and a cluster initially seeded by the mean of $X_s$ will be labeled as the u.u. class. 

Our primary experiment suggests that such an approach works equally well as the RTSCV method when the u.u.s consist of only one cluster (sub-class) and are far away from the known base classes. As illustrated in the top-left plot of Figure \ref{fig:csi}, in such a setting CSI has a very similar OSR performance as our RTSCV, under different covariance levels of the u.u. class. Nevertheless, when the u.u.s form multiple clusters or are close to the base classes, the performance of CSI plunges significantly, as illustrated in the top-right and bottom plots of Figure \ref{fig:csi}. This is possibly because of the large inconsistency between the mean of $X_s$ as the initial seed of the u.u. class and the true u.u.s distribution. Bottom: There is one u.u. cluster, whose distance to the known base classes is altered across trials.

In response to that, one potential improvement of the CSI method might be to incorporate some priors on the distribution of the u.u. class, i.e., the number of sub-classes or the means of them, with light involvement of human experts. We believe that this is a very promising direction for future works.

\newpage
\section{Plots of RTSCV Decision Boundaries}
\subsection{RTSCV+SVM Decision Boundaries}

\begin{figure}[h!]
\centering
\includegraphics[width=0.31\textwidth]{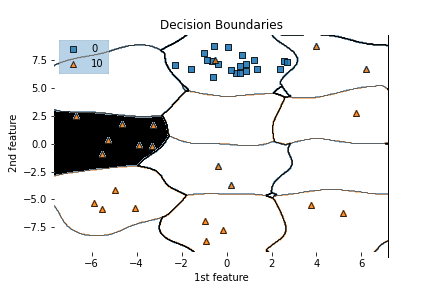}
\includegraphics[width=0.31\textwidth]{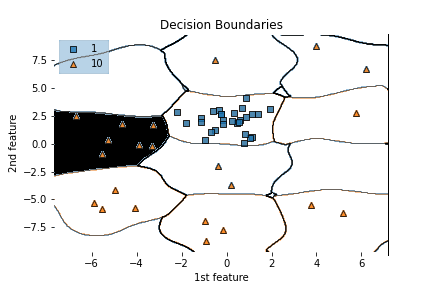}
\includegraphics[width=0.31\textwidth]{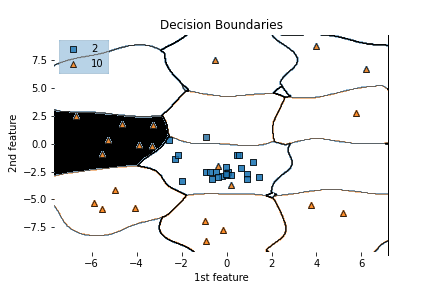}
\includegraphics[width=0.31\textwidth]{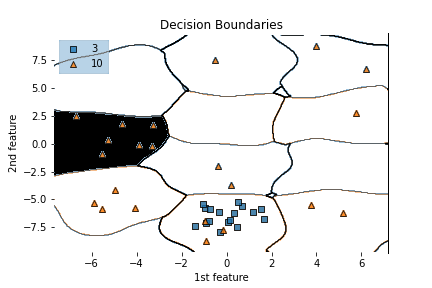}
\includegraphics[width=0.31\textwidth]{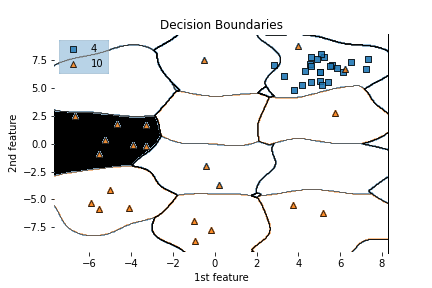}
\includegraphics[width=0.31\textwidth]{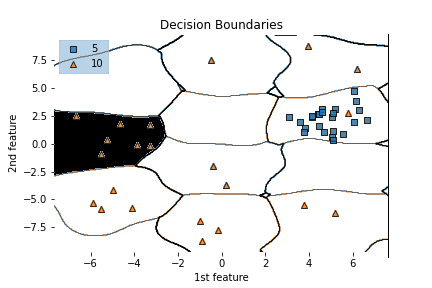}
\includegraphics[width=0.31\textwidth]{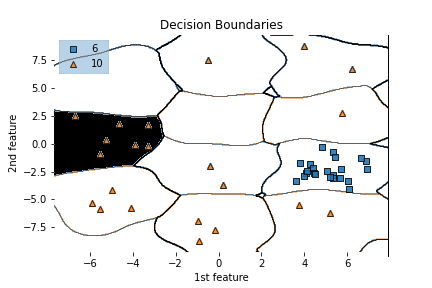}
\includegraphics[width=0.31\textwidth]{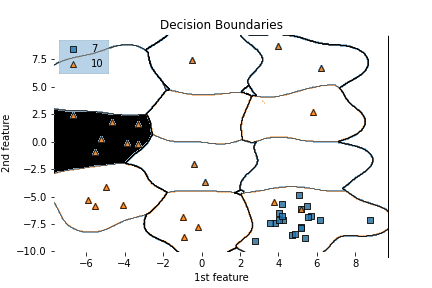}
\includegraphics[width=0.31\textwidth]{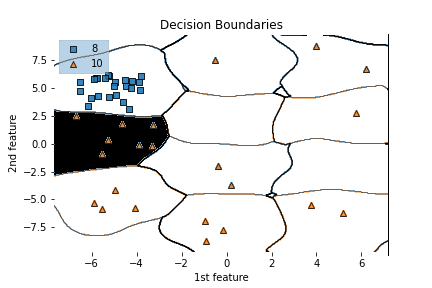}
\includegraphics[width=0.31\textwidth]{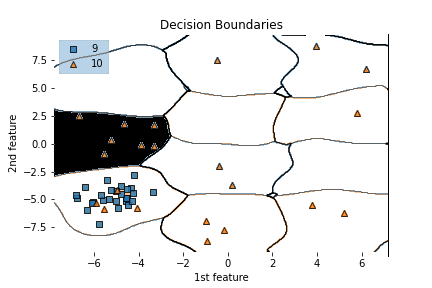}
\caption{RTSCV decision boundaries after cross-validation using SVM, fitted on the entire augmented training set consisting of 10 known classes and the sample class $X_s$. The black region represents the dummy (u.u.s) class where we intend to trap the u.u.s. Points from the sample class are represented by triangles with pseudo-label 10 while points from one of the known classes are represented by squares with the respective class label.}
\end{figure}

\newpage
\subsection{RTSCV+KNN Decision Boundaries}
\begin{figure}[h!]
\centering
\includegraphics[width=0.31\textwidth]{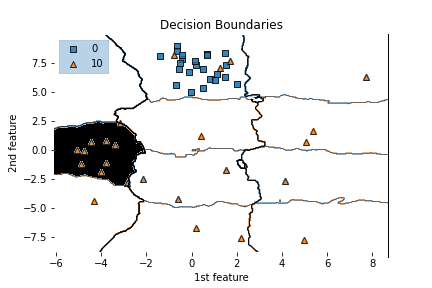}
\includegraphics[width=0.31\textwidth]{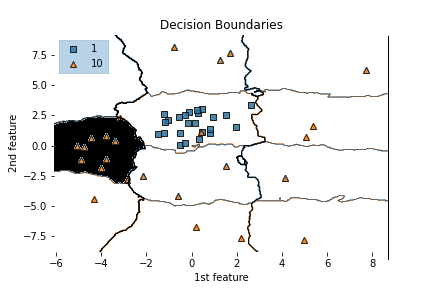}
\includegraphics[width=0.31\textwidth]{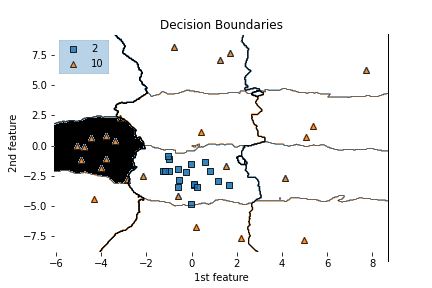}
\includegraphics[width=0.31\textwidth]{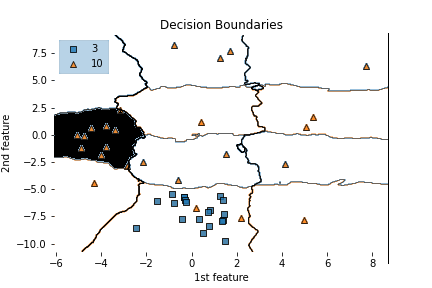}
\includegraphics[width=0.31\textwidth]{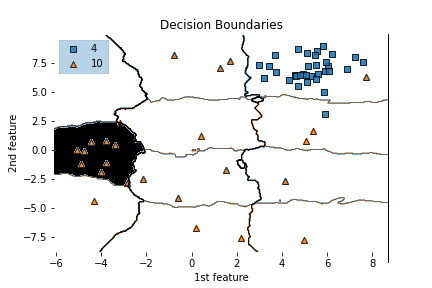}
\includegraphics[width=0.31\textwidth]{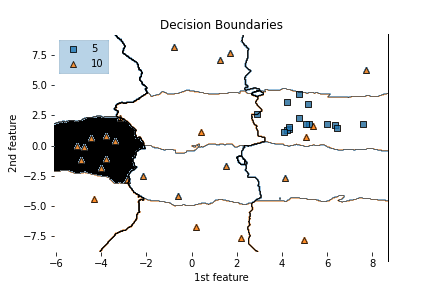}
\includegraphics[width=0.31\textwidth]{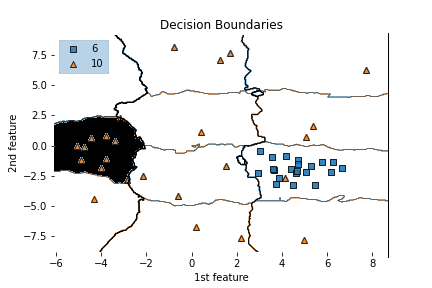}
\includegraphics[width=0.31\textwidth]{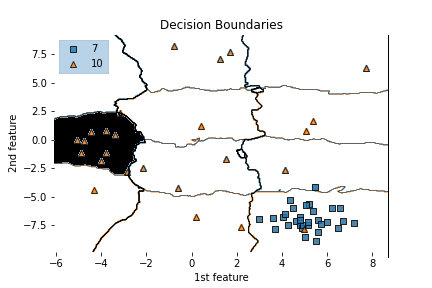}
\includegraphics[width=0.31\textwidth]{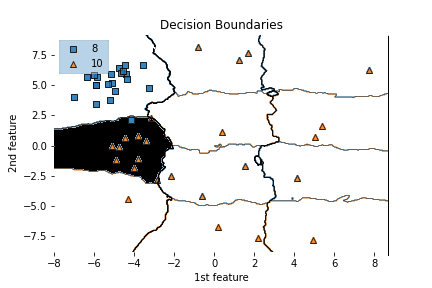}
\includegraphics[width=0.31\textwidth]{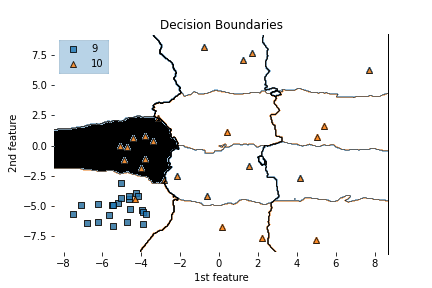}
\caption{RTSCV decision boundaries after cross-validation using KNN, fitted on the entire augmented training set consisting of 10 known classes and the sample class $X_s$. The black region represents the dummy (u.u.s) class where we intend to trap the u.u.s. Points from the sample class are represented by triangles with pseudo-label 10 while points from one of the known classes are represented by squares with the respective class label.}
\end{figure}

\newpage
\subsection{RTSCV+Decision Tree Decision Boundaries}
\begin{figure}[h!]
\centering
\includegraphics[width=0.31\textwidth]{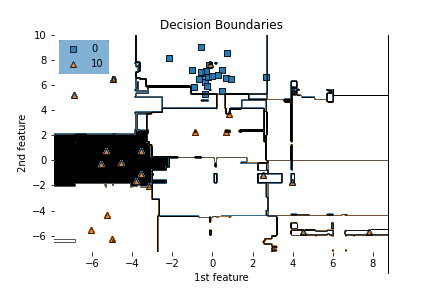}
\includegraphics[width=0.31\textwidth]{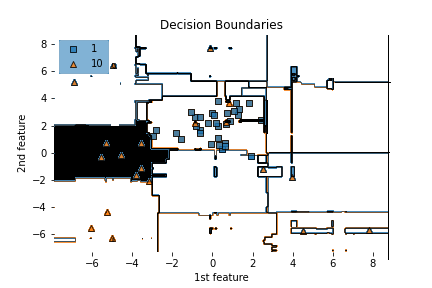}
\includegraphics[width=0.31\textwidth]{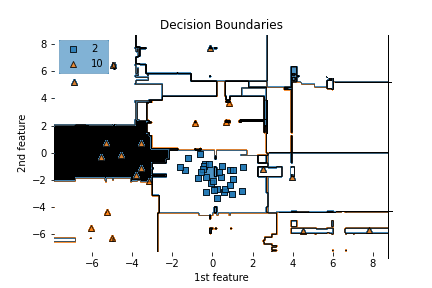}
\includegraphics[width=0.31\textwidth]{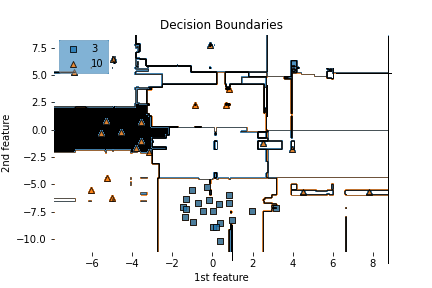}
\includegraphics[width=0.31\textwidth]{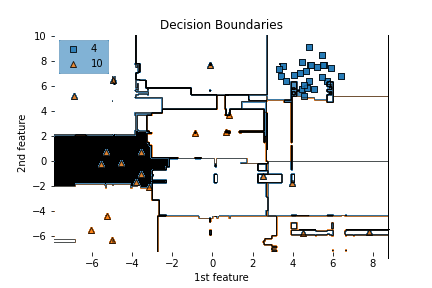}
\includegraphics[width=0.31\textwidth]{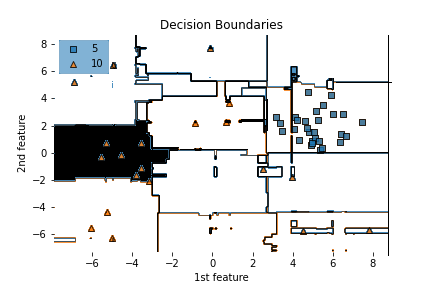}
\includegraphics[width=0.31\textwidth]{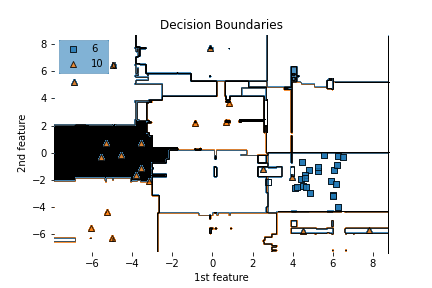}
\includegraphics[width=0.31\textwidth]{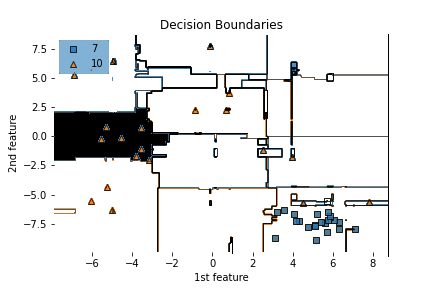}
\includegraphics[width=0.31\textwidth]{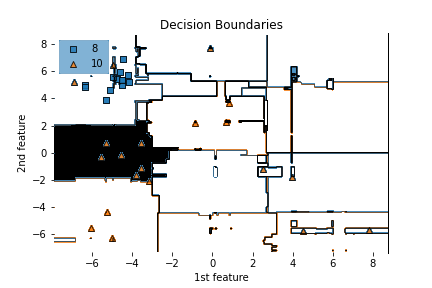}
\includegraphics[width=0.31\textwidth]{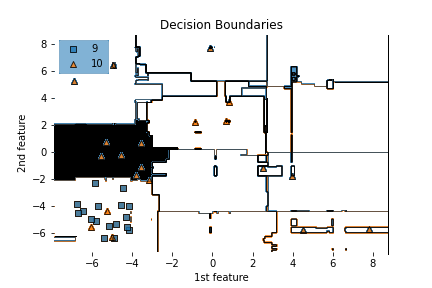}
\caption{RTSCV decision boundaries after cross-validation using Decision Tree, fitted on the entire augmented training set consisting of 10 known classes and the sample class $X_s$. The black region represents the dummy (u.u.s) class where we intend to trap the u.u.s. Points from the sample class are represented by triangles with pseudo-label 10 while points from one of the known classes are represented by squares with the respective class label.}
\end{figure}

\end{document}